\documentclass{article} 

\usepackage{times}  
\usepackage{helvet}  
\usepackage{courier}  
\usepackage[hyphens]{url}  
\usepackage{graphicx} 
\usepackage{natbib}  
\usepackage{caption} 
\usepackage{algorithm}
\usepackage{algorithmic}

\usepackage{newfloat}
\usepackage{listings}
\DeclareCaptionStyle{ruled}{labelfont=normalfont,labelsep=colon,strut=off} 
\floatstyle{ruled}
\newfloat{listing}{tb}{lst}{}
\floatname{listing}{Listing}
\usepackage[utf8]{inputenc} 
\usepackage{multirow}
\usepackage{todonotes}
\usepackage[T1]{fontenc}    
\usepackage{hyperref}       
\usepackage{url}            
\usepackage{booktabs}       
\usepackage{amsfonts}       
\usepackage{nicefrac}       
\usepackage{microtype}      
\usepackage{xcolor}         
\usepackage{amsmath,amsthm,amssymb}
\usepackage{bm}
\usepackage{pifont}
\usepackage{soul}
\usepackage{authblk}
\usepackage[inline]{enumitem}

\newcommand{\rsqrt}[1]{#1^{-\frac{1}{2}}}
\newcommand{\RR}{I\!\!R} 
\newcommand{\NN}{I\!\!N} 
\newcommand{\CC}{I\!\!\!\!C} 
\DeclareMathOperator{\Diag}{Diag}

\DeclareMathOperator{\sgn}{sgn}

\newtheorem{theorem}{Theorem}
\newtheorem{corollary}{Corollary}
\newcommand{\xmark}{\ding{53}}%
\newcommand{\ii}{\mathbf{i}}
\newcommand{\signum}{SigMaNet}
\newcommand{\rete}{SigMaNet}
\newcommand{\signumlaplacian}{Sign-Magnetic Laplacian}
\newcommand{\laplaciano}{Sign-Magnetic Laplacian}

\usepackage{tikz}
\def\checkmark{\tikz\fill[scale=0.4](0,.35) -- (.25,0) -- (1,.7) -- (.25,.15) -- cycle;}

\title{SigMaNet: One Laplacian to Rule Them All}
\author{
    Stefano Fiorini,\textsuperscript{\rm 1}
    Stefano Coniglio,\textsuperscript{\rm 2}
    Michele Ciavotta,\textsuperscript{\rm 1}
    Enza Messina\textsuperscript{\rm 1}
}
\date{%
\small
    $^1$University of Milano-Bicocca, Milan, Italy\\%
    $^2$University of Bergamo, Bergamo, Italy\\[2ex]%
        stefano.fiorini@unimib.it, stefano.coniglio@unibg.it, michele.ciavotta@unimib.it, enza.messina@unimib.it
}
\usepackage{geometry}
\begin{document}

\maketitle

\begin{abstract}
  This paper introduces \rete{}, a generalized Graph Convolutional Network (GCN) capable of handling both undirected and directed graphs with weights not restricted in sign nor magnitude. The cornerstone of \rete{} is the \textit{\laplaciano{}} ($L^{\sigma}$), a new Laplacian matrix that we introduce {\em ex novo} in this work.
$L^{\sigma}$ allows us to bridge a gap in the current literature by extending the theory of spectral GCNs to (directed) graphs with both positive and negative weights. $L^{\sigma}$ exhibits several desirable properties not enjoyed by other Laplacian matrices on which several state-of-the-art architectures are based, among which encoding the edge direction and weight in a clear and natural way that is not negatively affected by the weight magnitude.
$L^{\sigma}$ is also completely parameter-free, which is not the case of other Laplacian operators such as, e.g., the \textit{Magnetic Laplacian}.
The versatility and the performance of our proposed approach is amply demonstrated via computational experiments. Indeed, our results show that, for at least a metric, \rete{} achieves the best performance in 15 out of 21 cases and either the first- or second-best performance in 21 cases out of 21, even when compared to architectures that are either more complex or that, due to being designed for a narrower class of graphs, should---but do not---achieve a better performance.

\end{abstract}




\section{Introduction}\label{sec:intro}

The dramatic advancements of neural networks and deep-learning have provided researchers and practitioners with extremely powerful analytics tools. 
Increasingly complex phenomena and processes which can often be modeled as graphs or networks, such as, e.g., social networks~\citep{backstrom2011supervised}, knowledge graphs~\citep{zou2020survey}, protein interaction networks~\citep{kashyap2018protein}, or the World Wide Web (only to mention a few) can now be successfully addressed via Graph Convolutional Networks (GCNs).
Compared with other approaches, GCNs effectively manage to represent the data and its interrelationships by explicitly capturing the topology of the underlying graph within a suitably-designed convolution operator.

In the literature, GCNs mostly belong to two categories: spectral based and spatial based~\citep{wu2020comprehensive}.
Spatial GCNs define the convolution operator as a localized aggregation operator~\citep{wang2019dynamic} (although, rigorously speaking, such an operator may not always be called a convolution operator from a mathematical perspective).
Differently, spectral GCNs define the convolution operator (a rigorous one in this case) as a function of the eigenvalue decomposition of the Laplacian matrix associated with the graph~\citep{kipf2016semi}.
In their basic definition, both methods assume the graph to be undirected and to feature nonnegative weights. 
For a comprehensive review, we refer the reader to~\citep{zhang2019graph, wu2020comprehensive, he2022pytorch}. 

Many real-world processes and phenomena can be modeled as directed graphs. While spatial GCNs have a natural extension to directed graphs, most spectral methods require the graph to be undirected and to feature nonnegative weights for their convolution operator to be well-defined. Indeed, due to being based on graph signal processing~\citep{sandryhaila2013discrete, chen2015discrete}, spectral GCNs require three fundamental properties to be satisfied which fail to hold when the graph is directed and/or features negative weights:
\textit{a)} the Laplacian matrix must be diagonalizable, i.e., it must admit an eigenvalue decomposition, 
\textit{b)} the Laplacian matrix must be positive semidefinite, and \textit{c)} the
spectrum of the normalized Laplacian matrix must be upper-bounded by 2~\citep{kipf2016semi,wu2020comprehensive}.

In recent years, several extensions of the definition of Laplacian matrix have been proposed to overcome the first limitation, i.e., to handle directed graphs (see, e.g., the {\em Magnetic Laplacian}~\citep{zhang2021magnet, zhang2021smgc} and the {\em Approximate Digraph Laplacian} constructed via the PageRank matrix~\citep{Tong2020}).
Differently, no spectral techniques have been introduced so far to overcome the second limitation, i.e., to handle graphs with edge weights unrestricted in sign, which arise in many relevant applications (such as, e.g., those where the graph models credit/debit transactions between customers, like/dislike evaluations among users, or positive/negative user opinions).
This paper aims at overcoming such a major limitation.

\paragraph{Main contributions and novelty of the work}
\begin{itemize}
\item We extend the theory of spectral-based GCNs by introducing \signum{}, the first spectral GCN capable of handling both directed and undirected graphs with weights unrestricted in sign and of arbitrary magnitude.
\item We introduce the {\em \signumlaplacian{}} matrix $L^{\sigma}$, a new Laplacian matrix (which can be of independent interest beyond the scope of this paper) upon which \rete{} is built. We show that $L^{\sigma}$ satisfies all the properties that are needed for the definition of a convolution operator, among which being positive semidefinite regardless of the sign and magnitude of the edge weights.
  $L^{\sigma}$ also exhibits useful structural properties, among which being positively homogeneous (thus being proportional to the magnitude of the graph weights) and allowing for a natural interpretation of the weight sign in terms of the edge direction---two properties that other Laplacian matrices do not enjoy. Compared with other proposals, the definition of $L^{\sigma}$ is also parameter-free.
\item We experiment with \rete{} on four tasks using both real-world and synthetically-generated datasets. For the first task, we consider graphs with positive and negative edge weights, for which \rete{} is the only spectral GCN that can be adopted, and find it to outperform the (more complex) state-of-the-art networks in 3 cases out of 5. For the other three tasks, we consider graphs with nonnegative weights, which allows us to compare \rete{} with the other state-of-the-art spectral GCNs. While these GCNs are designed solely for graphs with nonnegative weights and, thus, one may expect that the wider applicability of \rete{} may come at the price of an inferior performance, our experiments show that this is not the case, as \rete{} is found to outperform the other GCNs in 11 cases out of 16. Across all four tasks, \rete{} is consistently either the best performing or the second-best performing method according to at least a metric.
\end{itemize}


\section{Preliminaries and previous works}\label{sec:preliminaries}

For a given $n \in \NN$, we denote by $[n]$ the set of integers $\{1, \dots, n\}$. For a given matrix $M$ of appropriate dimensions with real eigenvalues, we denote its largest eigenvalue by $\lambda_\text{max}(M)$. Throughout the paper, $e$ and $e e^\top$ denote the all-one vector and matrix of appropriate dimensions. Undirected and directed graphs are denoted by $G = (V,E)$, where $V$ is the set of vertices and $E$ the set of edges. In the undirected (directed) case, $E$ is a collection of unordered (ordered) pairs of elements in $V$. $G$ is always assumed to be self-loop free.

\subsection{Generalized convolution matrices}\label{subsec:generalized-convolution}

Let $M \in \CC^{n \times n}$, $n \in \NN$, be a positive semidefinite Hermitian matrix\footnote{A matrix is called Hermitian if its real part is symmetric and its imaginary part skew-symmetric.} with eigenvalue (or spectral) decomposition $M = U \Lambda U^*$,
where $\Lambda \in \RR^{n \times n}$ is the diagonal matrix whose elements are the (real, as $M$ is Hermitian) eigenvalues of $M$, $U \in \CC^{n \times n}$, and $U^*$ is $U$'s complex conjugate.
For each $i \in [n]$, the $i$-th column of $U$ coincides with the $i$-th eigenvector of $M$ corresponding to its $i$-th eigenvalue $\Lambda_{ii}$. The columns of $U$ form a basis of $\CC^n$. We assume $\lambda_\text{max}(M) \leq 2$.

Given a {\em signal} $x \in \CC^n$, let $\widehat{x}$ be its discrete Fourier transform 
with basis $U$, i.e., $\widehat{x} = U^* x$.
As $U^{-1} = U^*$, the transform is invertible and the inverse transform of $\widehat{x}$ reads $x = U \widehat{x}$.
Given a {\em filter} $y \in \CC^n$, the transform of its convolution with $x$ satisfies the relationship $\widehat{y * x} = \widehat{y} \odot \widehat{x} = \Diag(\widehat{y}) \widehat{x}$, where $*$ and $\odot$ denote the convolution and the Hadamard (or component-wise) product, respectively.
Applying the inverse transform, we have $y * x = U \Diag(\widehat{y}) U^* x$. Letting $\Sigma := \Diag(\widehat{y})$, we call a {\em generalized convolution matrix} the matrix $Y := U \Sigma U^*$, as $y * x = Y x$. 

Let $\tilde \Lambda := \frac{2}{\lambda_\text{max}(M)} \Lambda - I$ be the normalization of $\Lambda$. As $U U^* = I$, the same normalization applied to $M$ leads to $\tilde M =  U \tilde \Lambda U^* = \frac{2}{\lambda_\text{max}} M - I$.
%
Following~\cite{hammond2011wavelets,kipf2016semi}, we assume that $y$ is such that the entries of $\widehat{y}$ are real-valued polynomials in $\tilde \Lambda$, i.e., that $\widehat{y}_i = \sum_{k=0}^K \theta_k T_k (\tilde \lambda_i), i \in [n]$, where $\theta_0, \dots, \theta_K\in \RR$, $K \in \NN$, and $T_k$ is the Chebyshev polynomial of the first kind of order $k$. $T_k$ is recursively defined as $T_0(x) = 1, T_1(x) = x$, and $T_k(x) = 2x T_{k-1}(x) - T_{k-2}(x)$ for $k \geq 2$, with $x \in \RR \cap [-1,1]$.
Thus, we rewrite $\Sigma$ as $\Sigma = \Diag(\widehat{y}) = \sum_{k=0}^K \theta_k T_k(\tilde \Lambda)$, where $T_k$ is applied component-wise to $\tilde \Lambda$, i.e., $(T_k(\tilde \Lambda))_{ij} = T_k(\tilde \Lambda_{ij})$ for all $i,j \in [n]$. With this, the convolution of $x$ by $y$ can be rewritten as $Yx = U \Sigma U^* x = U \left(\sum_{k=0}^K \theta_k T_k(\tilde \Lambda) \right) U^*x$. 
Since, as it is easy to verify, $(U \tilde \Lambda U^*)^k = U \tilde \Lambda^k U^*$ holds for all $k \in \NN$, one can also verify that $Yx = U \left(\sum_{k=0}^K \theta_k T_k(\tilde \Lambda) \right) U^*x = \sum_{k=0}^K \theta_k T_k(U \tilde \Lambda U^*) x = \sum_{k=0}^K \theta_k T_k(\tilde M) x$.

Assuming $\lambda_\text{max} = 2$, we have $\tilde M = M - I$. Letting $K = 1$ and $\theta_1 = -\theta_0$, deduce:
\begin{equation}\label{eq:Y-convolution}
    y*x = Yx = (\theta_0 I - \theta_0 (M - I)) x = \theta_0 (2I - M) x.
\end{equation}
%
If $M$ is chosen so as to express the topology of the graph and $x$ coincides with the graph features, Eq.~\eqref{eq:Y-convolution} represents the convolution operation underlying a spectral GCN.
$M$ should satisfy three properties for Eq.~\eqref{eq:Y-convolution} to apply: 
\begin{enumerate*}[label=\textit{\roman*)},series=MyList, before=\hspace{-0.3ex}]
    \item it should admit an eigenvalue decomposition, \label{prop1}
    \item it should be positive semidefinite, and \label{prop2}
    \item its spectrum should be upper-bounded by 2. \label{prop3}
\end{enumerate*}
Examples of $M$ are given in the following subsections. 


\subsection{Spectral convolutions for undirected graphs}\label{sec:spectralconv}

Let $G=(V,E)$ be an undirected graph with $n = |V|$ without weights nor signs associated with its edges and let $A \in \{0,1\}^{n \times n}$ be its adjacency matrix, with $A_{ij} = 1$ if and only if $\{i,j\} \in E$. The Laplacian matrix of $G$ is defined as:
$$
L := D - A,
$$
where $D := \Diag(A e)$ is a diagonal matrix and, for each $i \in V$, $D_{ii}$ equal to the degree of node~$i$~\citep{chung1997spectral}.
The normalized Laplacian matrix is defined as:
$$
L_\text{norm} := \rsqrt{D} L \rsqrt{D} = \rsqrt{D} (D - A) \rsqrt{D} = I - \rsqrt{D} A \rsqrt{D}.
$$
$L_\text{norm}$ satisfies many properties, among which~\ref{prop1}, \ref{prop2} and \ref{prop3}.

The spectral convolution on the undirected graph $G$ introduced by~\citet{kipf2016semi} is obtained by letting $M := L_\text{norm}$ and defining $Y$ as done before.
Eq.~\eqref{eq:Y-convolution} becomes:
\begin{align}\label{eq:beforetrick}
\nonumber
   y * x= Y x & = \theta_0 (2 I - (I - \rsqrt{D} A \rsqrt{D}) x \\
   & = \theta_0 (I + \rsqrt{D} A \rsqrt{D}) x.
\end{align}
To alleviate numerical instabilities and exploding/vanishing gradients when training a GCN built on Eq.~\eqref{eq:beforetrick}, \citet{kipf2016semi} suggest the adoption of the following modified equation with a modified convolution matrix $\tilde Y$:
\begin{equation}\label{eq:aftertrick}
    y * x = \tilde Y x = \theta_0 (\rsqrt{\tilde D} \tilde A \rsqrt{\tilde D}) x,
\end{equation}
where $\tilde A := A + I$ and $\tilde D := \Diag(\tilde A e)$.

\paragraph{Drawbacks} 
The Laplacian matrix $L$ is well defined only for undirected graphs with (if any) nonnegative weights for two reasons.
$i$)~If $G$ is a directed graph, the adjacency matrix $A$ (and, thus, $L$) is, in the general case, not symmetric; thus, $L$ may not admit an eigenvalue decomposition.
$ii$)~If $G$ is directed, the sum of the columns of $A$ (out degree) is not necessarily identical to the sum of its rows (in degree); thus, the matrix $D$ is not well defined (as $\Diag(A e) \neq \Diag(e^\top A)$).
If $G$ is undirected but features edges $\{i,j\} \in E$ with a negative weight $w_{ij} <0$, $L$ is not necessarily positive semidefinite as, even if $\Diag(A e) = \Diag(e^\top A)$, $\rsqrt{D} \notin \RR^{n \times n}$ if $D_{ii} < 0$.

\subsection{Extending spectral convolutions to directed graphs}\label{sec:drawback}

Since the adjacency matrix of a directed graph is asymmetric, the Laplacian matrix $L$ defined before does not enjoy properties~\ref{prop1}, \ref{prop2} and \ref{prop3} and, therefore, it is not possible to directly apply Eq.~\eqref{eq:beforetrick} to define a spectral graph convolution.
Alternative approaches such as those of~\citet{tong2020directed} and~\citet{Tong2020} (which split the adjacency matrix $A$ into a collection of symmetric matrices in such a way that the information regarding the direction of the edges is not lost) are known, but they typically come at the cost of increasing the size and complexity of the neural network.

A more direct way to encode the directional information of the edges is resorting to complex-valued matrices that are Hermitian. Indeed, albeit asymmetric in the general case, Hermitian matrices admit an eigenvalue decomposition with real eigenvalues. 
The {\em Magnetic Laplacian} is one such matrix. It was first introduced in particle physics and quantum mechanics by~\citet{lieb1993fluxes} and then applied in the context of community detection by~\citet{fanuel2018magnetic}, in graph signal processing by~\citet{furutani2019graph} and, lastly, in the context of spectral GCNs by~\citet{zhang2021magnet, zhang2021smgc}.

Let $A_s := \frac{1}{2} \left(A+A^\top\right)$ be the symmetrized version of $A$ and let $D_s := \Diag(A_s e)$. The {\em Magnetic Laplacian} is defined as the following Hermitian
positive semidefinite matrix:
$$ L^{(q)} := D_s - H^{(q)}, \qquad \qquad \text{with}$$
$$H^{(q)} := A_s \odot \exp \left(\ii \, \Theta^{(q)} \right),
    \Theta^{(q)} := 2 \pi q\left(A-A^\top \right),$$
    where $\ii$ is the square root of the negative unit, i.e., $\ii = \sqrt{-1}$, and $\exp \left(\ii \, \Theta^{(q)} \right) = \cos(\Theta^{(q)}) + \ii \sin(\Theta^{(q)})$, where $\cos(\cdot)$ and $\sin(\cdot)$ are applied component-wise. $\Theta$ is a phase matrix that captures the directional information of the edges.
    The parameter $q \geq 0$ represents the electric charge. It is typically set to values smaller than 1 such as $\left[0, \frac{1}{4}\right]$ as in~\citet{zhang2021magnet} or $\left[0, \frac{1}{2}\right]$ as in~\citet{mag2017}.
If $q = 0$, $\Theta^{(q)} = 0$ and $L^{(q)}$ boils down to the Laplacian matrix $L$ defined on the "symmetrized" version of the graph with adjacency matrix $A_s$ (which, crucially, renders $G$ completely undirected and its directional information is lost).


For unweighted directed graphs where $A \in \{0,1\}^{n \times n}$, $H^{(q)}$ straightforwardly captures the graph's directional information. Assuming $q = 0.25$, we have $H^{(q)}_{ij} = H^{(q)}_{ji} = 1 + \ii 0$ if $(i,j), (j,i) \in E$ and $H^{(q)}_{ij} = 0 + \ii \frac{1}{2}$ and $H^{(q)}_{ji} = 0 - \ii \frac{1}{2}$ if $(i,j) \in E  \wedge (j,i) \notin E$. This way, digons (pairs of antiparallel edges) are represented as single undirected edges in the real part of $H^{(q)}$ whereas any other edge is represented in the imaginary part of $H^{(q)}$ with a sign encoding its direction.

\paragraph{Drawbacks}
The {\em Magnetic Laplacian} $L^{(q)}$ suffers from two drawbacks.
Drawback \#1 is that $L^{(q)}$ is well defined only for graphs with nonnegative weights. Indeed, if $(D_s)_{ii} < 0$ for some $i \in V$, in the general case $L^{(q)}$ is not positive semidefinite and $\rsqrt{D_s}$ does not belong to $\RR^{n \times n}$.
%
Drawback \#2 is that, even when restricted to graphs with nonnegative weights, $L^{(q)}$ exhibits a crucial sign-pattern inconsistency if the edge weights are sufficiently large. Indeed, while for unweighted graphs $L^{(q)}$ always captures the directional information of the edges by the sign of the imaginary part of $H^{(q)}$, this
is not necessarily the case for weighted graphs, where the sign pattern of $H^{(q)}$ can drastically change irrespective of the edge direction by just scaling the edge weights by a positive constant.
To see this, assume, for instance, $(i,j) \in E$ and $(j,i) \notin E$ with $A_{ij} = 1$. Then, we obtain: $H^{(0.25)}_{ij} = 0.40 \cdot 0.31 + \ii 0.40 \cdot 0.95$ and  $H^{(0.25)}_{ji} = 0.40 \cdot 0.31 - \ii 0.40 \cdot 0.95$ by scaling $A_{ij}$ by 0.8; $H^{(0.25)}_{ij} = -1 + \ii 0$ by scaling $A_{ij}$ by 2; $H^{(0.25)}_{ij} = 0 + \ii \frac{5}{2}$ by scaling $A_{ij}$ by 5; and $H^{(0.25)}_{ij} = \frac{36}{2} + \ii 0$ by scaling $A_{ij}$ by 36. This shows that $L^{(q)}$ is not robust to scaling and that, in it, the edge direction information can easily be lost.
%


\section{Our proposal: the \signumlaplacian{} and \rete{}}\label{sec:signum}

In this section, we extend the theory underlying spectral GCNs by introducing the {\em \signumlaplacian} matrix, a positive semidefinite Hermitian matrix that well captures the directional as well as the weight information of any directed graph with weights unrestricted in sign nor magnitude without suffering from the two drawbacks we outlined before.


\subsection{\signumlaplacian}\label{subsec:laplaciano}

We introduce the following Hermitian matrix, which we refer to as the {\em \signumlaplacian{}}:
$$
L^{\sigma} := \bar D_s - H^{\sigma}, \qquad\qquad \text{with}
$$
$$
H^{\sigma} := A_s \odot \Big( ee^\top  - \sgn (|A - A^\top|) + \ii \sgn \big(|A| - |A^\top| \big) \Big),
$$
where $A_s := \frac{1}{2} \left(A + A^\top \right)$, $\bar D_s := \Diag(|A_s| \, e)$, and $\sgn: \RR \rightarrow \{-1,0,1\}$ is the signum function (applied component-wise).
Let us illustrate the way the graph topology and its weights are stored in $H^{\sigma}$.
$H^{\sigma}$ encodes the direction and weight of every edge $(i,j) \in E$ that does not have an antiparallel edge $(j,i)$ purely in its imaginary part by $H^{\sigma}_{ij} = - H^{\sigma}_{ji} = 0 + \ii \frac{1}{2} A_{ij}$.
Pairs of antiparallel edges with $A_{ij} = A_{ji}$ are encoded purely in the real part by $H^{\sigma}_{ij} = H^{\sigma}_{ji} = \frac{1}{2}(A_{ij} + A_{ji}) + \ii 0$ (as if they coincided with an undirected edge of the same weight).
Differently, pairs of antiparallel edges with $A_{ij} \neq A_{ji}$ are encoded purely in the imaginary part by $H^{\sigma}_{ij} = -H^{\sigma}_{ji} = 0 + \ii \frac{1}{2} (A_{ij} + A_{ji})$ if $|A_{ij}| > |A_{ji}|$ and $H^{\sigma}_{ij} = - H^{\sigma}_{ji} = 0 - \ii \frac{1}{2} (A_{ij} + A_{ji})$ if $|A_{ij}| < |A_{ji}|$.

We define the normalized version of $L^{\sigma}$ as:
\begin{equation}\label{eq:norma}
  L^{\sigma}_\text{norm} := \rsqrt{\bar D_s} L^{\sigma} \rsqrt{\bar D_s} = I - \rsqrt{\bar D_s} H^{\sigma} \rsqrt{\bar D_s}.
\end{equation}

One can show that both $L^{\sigma}$ and $L^{\sigma}_\text{norm}$ are Hermitian by construction. Therefore, they admit an eigenvalue decomposition and, thus, satisfy property~\ref{prop1}.

$L^{\sigma}$ is defined in such a way that, if $G$ is unweighted, it mirrors the behavior of $L^{(q)}$ with $q = 0.25$:
\begin{theorem}
If $A \in \{0,1\}^{n \times n}$ and $q=0.25$, $L^{\sigma} = L^{(q)}$.
\end{theorem}
\noindent In contrast with $L^{(q)}$, $L^{\sigma}$ does not suffer from drawback \#1 as it is well-defined even when $G$ features negative weights.

With the following two results, we show that $L^{\sigma}$ and $L^{\sigma}_\text{norm}$ enjoy the two remaining properties \ref{prop2} and \ref{prop3} that are required for the construction of a convolution operator:
\begin{theorem} \label{thm:psd}
$L^{\sigma}$ and $L^{\sigma}_{norm}$ are positive semidefinite.
\end{theorem}
\begin{theorem} \label{thm:eigenvalues}
$\lambda_\text{max}(L_{\text{norm}}^{\sigma}) \leq 2$.
\end{theorem}

With the next result, we show that $L^{\sigma}$ encodes the topology of $G$ (including its directions) and the weights of its edges in such a way that it is always proportional to the magnitude of $A$ (i.e., to the magnitude of graph weights):
\begin{theorem}\label{prop:homonegeneity}
Given a constant $\alpha \in \RR^+$, $L^{\sigma}$ satisfies the following positive homogeneity property:
$$
L^{\sigma}(\alpha A) = \alpha L^{\sigma}(A),
$$
where $L^{\sigma}(\alpha A)$ and $L^{\sigma}(A)$ are the \laplaciano{} matrices of a directed graph with, respectively, adjacency matrix $\alpha A \in \RR^{n \times n}$ and $A \in \RR^{n \times n}$.
\end{theorem}
Theorem~\ref{prop:homonegeneity} shows that $L^{\sigma}$ is robust to scaling applied to the weights of $G$. From it, we deduce the following result, which shows that $L^\sigma$ does not suffer from drawback \#2:
\begin{corollary}
The sign-pattern of $L^{\sigma}$ is uniquely determined by the topology of $G$ and, thus, $L^{\sigma}$ does not suffer from any sign-pattern inconsistencies.
\end{corollary}

Lastly, we show that $L^{\sigma}$ satisfies the following invariant:
\begin{theorem}\label{prop:sign-direction-invariant}
  Consider a weighted digon-free directed graph $G=(V,E)$. Given a directed edge $(i,j) \in E$ of weight $w_{ij}$, let $G'=(V,E')$ be a graph obtained by reversing the direction of $(i,j)$ in $G$ into $(j,i)$ and flipping the sign of its weight by letting $w_{ji} = -w_{ij}$. Let $L^{\sigma}(G)$ and $L^{\sigma}(G')$ be the $L^{\sigma}$ matrix defined on $G$ and $G'$, respectively. Then:
  $$
  L^{\sigma}(G) = L^{\sigma}(G').
  $$
\end{theorem}

Theorem~\ref{prop:sign-direction-invariant} shows that the behavior of $L^{\sigma}$ is consistent with applications where the graph models a flow relationship in which flipping the sign of an edge coincides with flipping its direction. 
This applies to, among others, scenarios where the weights represent flow values such as cash flows, where it is reasonable to assume that a negative flow from $i$ to $j$ correspond to a positive flow from $j$ to $i$.

\subsection{\rete{}'s architecture}\label{sub:SigMaNet}

As previously discussed, for the spectral convolution operator to be well defined the Laplacian matrix must satisfy properties~\ref{prop1},\ref{prop2}, and \ref{prop3}.
As the hermiticity of $L^{\sigma}_{\text{norm}}$,  Theorem~\ref{thm:psd}, and Theorem~\ref{thm:eigenvalues} show that $L^{\sigma}_{\text{norm}}$ enjoys these properties, Eq.~\eqref{eq:beforetrick} can be rewritten as:
$$
    Y x = \theta_0\left(I + \rsqrt{\bar D_s} H^{\sigma} \rsqrt{\bar D_s} \right)x.
$$

Following~\citet{kipf2016semi} to avoid numerical instabilities, we apply Eq.~\eqref{eq:aftertrick} with $\rsqrt{\tilde D_s}\tilde H^{\sigma} \rsqrt{\tilde D_s}$ {\em in lieu} of $I + \rsqrt{\bar D_s}H^{\sigma} \rsqrt{\bar D_s}$, where $\tilde H^{\sigma}$ and $\tilde D_{s}$ are defined based on $\tilde A := A + I$ rather than $A$.
We generalize the feature vector signal $x \in \CC^{n \times 1}$ to a feature matrix signal $X \in \CC^{n \times c}$ with $c$ input channels (i.e., a $c$-dimensional feature vector for every node of the graph).
Letting $\Theta \in \CC ^{c \times f}$ be a matrix of learnable filter parameters with $f$ filters and $\phi$ be an activation function applied component-wise to the input matrix, the output $Z^{\sigma} \in \CC ^{n \times f}$ of \rete{}'s convolutional layer is:
\begin{equation}\label{eq:convolution}
    Z^{\sigma} (X) = \phi(\rsqrt{\tilde D_s}\tilde H^{\sigma} \rsqrt{\tilde D_s}X\Theta).
\end{equation}
Since the argument of $\phi$ is a complex matrix and, thus, traditional activation functions cannot be directly adopted, we follow~\citet{zhang2021magnet} and rely on a complex version of the \textit{ReLU} activation function which is defined for a given $z \in \CC$ as $\phi(z) = z$ if $ \Re(z) \geq 0$ and $\phi(z) = 0$ otherwise. 
As the output of the convolutional layer $Z^{\sigma}$ is complex-valued, to coerce it into the reals without information loss we apply an \textit{unwind} operation by which $Z^{\sigma} (X)\in \CC^{n \times f}$ is transformed into $[\Re(Z^{\sigma} (X)); \Im(Z^{\sigma} (X)] \in \RR^{n \times 2f}$.
To obtain the final result based on the task at hand, we apply either a linear layer with weights $W$ or a $1D$ convolution. 

Considering, e.g., the task of predicting the class of an edge, \signum{} is defined as:
\begin{equation*}
    \text{softmax} \left(\text{unwind}\left( Z^{\sigma(2)} \left( Z^{\sigma(1)} \left(X^{(0)} \right)\right)\right)       W \right),
\end{equation*}
where $X^{(0)} \in \RR ^{n \times c}$ is the input feature matrix,
$Z^{\sigma(1)} \in \CC^{n \times f_1}$ and $Z^{\sigma(2)} \in \CC^{n \times f_2}$ are the spectral graph convolutional layers, $W \in \RR^{2f_2 \times d}$ are the weights of the linear layer (with $d$ being the number of classes), and $\text{softmax}: \RR^d \rightarrow \RR^d$ is the normalized exponential activation function.

\rete{} features a flexible architecture that differs from other spectral GCNs in the literature (e.g., MagNet~\citep{zhang2021magnet}) mainly in the way the convolutional layer is defined. As such, it can easily be applied to a variety of tasks in an almost task-agnostic way (provided that one defines a suitable loss function) while architectures such as MagNet are suitable only for tasks whose graph has nonnegative edge weights.
As $L^\sigma$ is entirely parameter-free, \rete{} does not require any fine-tunings to optimize the propagation of topological information through the network, differently from, e.g., DiGraph~\citep{Tong2020} and MagNet.

\subsection{Complexity of \rete{}}

Assuming, as done in our experiments, that \rete{} features two graph-convolutional layers with $f_1$ and $f_2$ filters, each defined as in Eq.~\eqref{eq:convolution} and $c$ features per node, the complexity of \rete{} is $O(nc(n+f_1)+n f_1(n+f_2)+\Gamma)$, where $\Gamma$ accounts for the complexity of the last (task-specific) layer. For the four tasks that we consider in the next section, we have $\Gamma = m^\text{train} f_2 d$ for the first three (link sign/direction/existence prediction), where $m^\text{train}$ is the number of edges in the training set and $d$ is the number of classes, and $\Gamma = n f_2 d$ for the last one (node classification). The complexity is quadratic in $n$ and, assuming $O(f_1) = O(f_2) = O(c) = O(d)$, also quadratic in the feature/class space.

We remark that, while enjoying a wider applicability due to being able to handle graphs with edge weights unrestricted in sign), \rete{} features half the weights of MagNet.


\section{Numerical experiments}\label{sec:experiment}

In this section, we report on a set of computational experiments carried out on four tasks: \textit{link sign prediction}, \textit{link existence prediction}, {\em link direction prediction}, and \textit{node classification}.
The experiments are conducted to assess the performance of \rete{} on graphs with weights unrestricted in sign on which no other spectral GCNs can be applied (link sign prediction) and to compare it to other state-of-the-art spectral and spatial approaches on graphs with nonnegative weights (link existence/direction prediction and node classification). The code is available on GitHub.\footnote{\url{https://github.com/Stefa1994/SigMaNet}.}

For the link sign prediction task, we compare \signum{} with three categories of methods:
\textit{i)} signed network embedding: SiNE~\citep{wang2017signed}, SIGNet~\citep{10.1007/978-3-319-93037-4_13}, BESIDE~\citep{chen2018bridge}; \textit{ii)} Feature Engineering: FeExtra~\citep{leskovec2010predicting}; and \textit{iii)} signed graph neural networks: SGCN~\citep{derr2018signed}, SiGAT~\citep{huang2019signed}, and SDGNN~\citep{huang2021sdgnn}.
For the link prediction and node classification tasks, we compare \rete{} with the following three categories of methods:
\textit{i)} spectral methods designed for undirected graph: ChebNet~\citep{defferrard2016convolutional}, GCN~\citep{kipf2016semi}; 
\textit{ii)} spectral methods designed for directed graphs: DGCN~\citep{tong2020directed}, DiGraph~\citep{Tong2020}, DiGCL~\citep{tong2021directed}, and MagNet~\citep{zhang2021magnet},
and \textit{iii)} spatial methods: APPNP~\citep{klicpera2018predict}, SAGE~\citep{hamilton2017inductive}, GIN~\citep{xu2018powerful}, GAT~\citep{velivckovic2017graph}, and SSSNET~\citep{he2022sssnet}.\footnote{
\cite{MSGNN2022} (appeared after the submission of this paper) reports experiments in which SigMaNet seems to perform poorly. In them, though, SigMaNet is used in a much simplied configuration than in our paper that only features four convolutional filters, whereas, in this work, the number of filters is chosen from $\{16, 32, 64\}$ via hyperameter optimization.}

Throughout this section's tables, the best results are reported in \textbf{boldface} and the second best are \ul{underlined}.
%

\subsection{Datasets}

We test \signum{} on six real-world datasets from the literature: {\tt Bitcoin-OTC} and {\tt Bitcoin Alpha}~\citep{7837846}; {\tt Slashdot} and {\tt Epinions}~\citep{leskovec2010signed}; {\tt WikiRfa}~\citep{west2014exploiting}; and {\tt Telegram}~\citep{bovet2020activity}. 
In order to better assess \rete{}'s performance as the density of the graph increases, in three tasks we also consider a synthetic set of graphs generated via a direct stochastic block model (DSBM) with (unlike what is done in~\cite{zhang2021magnet}) edge weights greater than 1.
These datasets are generated by varying: \textit{i)} the number of nodes $n$; \textit{ii)} the number of clusters $C$; \textit{iii)} the probability $\alpha_{ij}$ to create an undirected edge between nodes $i$ and $j$ belonging to different clusters ; \textit{iv)} the probability $\alpha_{ii}$ to create an undirected edge between two nodes in the same cluster, and \textit{v)} the probability $\beta_{ij}$ of an edge taking a certain direction. Each node is labeled with the index of the cluster it belongs to.
%

\subsection{Link sign prediction}\label{sub:sign}
The link sign prediction task is a classification problem designed for graphs with both positive and negative edge weights. It consists of predicting the sign of the edges in the graph  and, thus, for such task \rete{} is the only spectral-based GCN that can be used.

For this task, we adopt the {\tt Bitcoin Alpha}, {\tt Bitcoin-OTC}, {\tt WikiRfa}, {\tt Slashdot}, and {\tt Epinions} datasets, which are directed graphs with weights of unrestricted sign ( necessary for the task to be applicable) and of arbitrary magnitude, with the sole exception of the last two, whose weights satisfy $A \in \{-1,0, +1\}^{n \times n}$. In these five datasets, the classes of positive and negative weighted edges are imbalanced (i.e., nearly 80\% are positive edges).
The experiments are run with $k$-cross validation with $k=5$, reporting the average score obtained across the $k$ splits. Connectivity is maintained when building each training set by guaranteeing that the graph used for training in each fold contain a spanning tree.
Following~\cite{huang2021sdgnn}, we
we adopt a 80\%-20\% training-testing split.

The results are reported in Table~\ref{tab:signpred}.\footnote{Except for \rete{}, the results are taken from~\cite{huang2021sdgnn}. For SGCN, SiGAT, and SDGNN, we chose to report the results in~\cite{huang2021sdgnn} rather than those in~\cite{he2022pytorch} as the former are better, and, thus, more challenging for \rete{}.
} 
We observe that \rete{} clearly outperforms all competitors on the three datasets whose graphs have unrestricted weights, i.e., {\tt Bitcoin Alpha}, {\tt Bitcoin-OTC}, and {\tt WikiRfa}.
On graphs with unit weights, i.e., {\tt Slashdot} and {\tt Epinions}, its performance, while marginally worse, is still in line with the best methods. This suggests the relevance of Corollary~1 towards \rete{}'s performance.
We remark that the latter is achieved in spite of \rete{} being less complex than the deep neural networks we compared it to here, which feature two sequentially-applied neural networks (one producing a set of embeddings from which the other one predicts the link sign via a logistics regression).

\begin{table*}[htb]
\scriptsize
\centering
\caption{Link sign prediction results assessed with four metrics}
\label{tab:signpred}
\begin{tabular}{cccccccccc}
\hline
                                                    &           & \multicolumn{3}{c}{\begin{tabular}[c]{@{}c@{}}Signed\\ Network Embedding\end{tabular}} & \begin{tabular}[c]{@{}c@{}}Feature\\ Engineering\end{tabular} & \multicolumn{4}{c}{\begin{tabular}[c]{@{}c@{}}Graph \\ Neural Network\end{tabular}} \\ \hline
\multicolumn{1}{c}{Dataset}                        & Metric (\%)    & SiNE                      & SIGNet                    & BESIDE                          & FeExtra                                                       & SGCN           & SiGAT          & SDGNN                   & \signum{}                  \\ \hline
\multicolumn{1}{c}{\multirow{4}{*}{Bitcoin Alpha}} & Micro-F1  & 94.58                    & 94.22                    & 94.89                          & 94.86                                                        & 92.56         & 94.56         & \ul{94.91}            & \textbf{95.13}         \\
\multicolumn{1}{c}{}                               & Binary-F1 & 97.16                    & 96.96                    & 97.32                          & \ul{97.30}                                                  & 96.07         & 97.14         & 97.29                  & \textbf{97.44}         \\
\multicolumn{1}{c}{}                               & Macro-F1  & 68.69                    & 69.65                    & 73.00                          & 71.67                                                        & 63.67         & 70.26         & \ul{73.90}            & \textbf{74.69}         \\
\multicolumn{1}{c}{}                               & AUC       & 87.28                    & 89.08                    & 89.81                          & 88.82                                                        & 84.69         & 88.72         & \ul{89.88}            & \textbf{92.46}         \\ \hline
\multicolumn{1}{c}{\multirow{4}{*}{Bitcoin-OTC}}   & Micro-F1  & 90.95                    & 92.29                    & 93.20                          & \ul{93.61}                                                  & 90.78         & 92.68         & 93.57                  & \textbf{94.49}         \\
\multicolumn{1}{c}{}                               & Binary-F1 & 95.10                    & 95.81                    & 96.28                          & \ul{96.53}                                                  & 94.91         & 96.02         & 96.47                  & \textbf{97.02}         \\
\multicolumn{1}{c}{}                               & Macro-F1  & 68.05                    & 73.86                    & 78.43                          & 78.26                                                        & 73.06         & 75.33         & \ul{80.17}         & \textbf{ 80.53}         \\
\multicolumn{1}{c}{}                               & AUC       & 85.71                    & 89.35                    & \ul{91.52}                    & 91.21                                                        & 87.55         & 90.55         & 91.24                  & \textbf{93.67}         \\ \hline
\multicolumn{1}{c}{\multirow{4}{*}{WikiRfa}}       & Micro-F1  & 83.38                    & 83.84                    & 85.89                          & 83.46                                                        & 84.89         & 84.57         & \ul{86.27}            & \textbf{86.56}         \\
\multicolumn{1}{c}{}                               & Binary-F1 & 89.72                    & 90.01                    & 91.17                          & 89.87                                                        & 90.69         & 90.42         & \ul{91.42}            & \textbf{91.64}         \\
\multicolumn{1}{c}{}                               & Macro-F1  & 73.19                    & 73.84                    & 78.03                          & 72.35                                                        & 75.27         & 75.35         & \ul{78.49}            & \textbf{78.66}         \\
\multicolumn{1}{c}{}                               & AUC       & 86.02                    & 86.82                    & \ul{89.81}                    & 86.04                                                        & 85.63         & 88.29         & 88.98                  & \textbf{90.53}         \\ \hline
\multicolumn{1}{c}{\multirow{4}{*}{Slashdot}}      & Micro-F1  & 82.65                    & 83.89                    & \ul{85.90}                          & 84.72                                                        & 82.96         & 84.94         & \textbf{86.16}            & {85.03}         \\
\multicolumn{1}{c}{}                               & Binary-F1 & 89.18                    & 89.83                    & \ul{91.05}                    & 90.70                                                        & 89.26         & 90.55         & \textbf{91.28}         & 90.59                  \\
\multicolumn{1}{c}{}                               & Macro-F1  & 72.73                    & 75.54                    & \textbf{78.92}                    & 73.99                                                        & 74.03         & 76.71         & \textbf{78.92}            & \ul{77.63}         \\
\multicolumn{1}{c}{}                               & AUC       & 84.09                    & 87.52                    & \textbf{90.17}                    & 88.80                                                        & 85.34         & 88.74         & 89.77                  & \ul{89.79}         \\ \hline
\multicolumn{1}{c}{\multirow{4}{*}{Epinions}}      & Micro-F1  & 91.73                    & 91.13                    & \ul{93.36}                          & 92.26                                                        & 91.12         & 92.93         & \textbf{93.55}            & 92.25         \\
\multicolumn{1}{c}{}                               & Binary-F1 & 95.25                    & 94.89                    & 96.15                    & \ul{95.61}                                                        & 94.86         & 95.93         & \textbf{96.28}         & {95.51}                  \\
\multicolumn{1}{c}{}                               & Macro-F1  & 81.60                    & 80.60                    & \ul{86.01}                          & 81.30                                                        & 81.05         & 84.54         & \textbf{86.10}            & 83.41         \\
\multicolumn{1}{c}{}                               & AUC       & 88.72                    & 90.95                    & 93.51                          & \ul{94.17}                                                  & 87.45         & 93.33         & 94.11                  & \textbf{94.19}         \\ \hline
\end{tabular}
\end{table*}

\subsection{Link (existence and direction) prediction}

We now consider two tasks: {\em existence prediction} and {\em direction prediction}. In the first one, the model is asked to predict whether $(u, v) \in E$ for a given pair of vertices $u, v\in V, u\neq v$ provided as input. The second one is a binary task where the model is asked to predict whether \textit{a)} $(u,v) \in E$ or \textit{b)} $(v,u) \in E$ or both.

For both tasks, we only consider graphs with nonnegative edge weights. This allows us to compare \rete{} not just to spatial GCNs as done before, but also to state-of-the-art spectral-based ones.
As such GCNs are designed solely for graphs with nonnegative weights, one may expect that the wider applicability of \rete{} should come at the price of inferior performances. Our experiments show that this is not the case.

The datasets that we consider are: {\tt Telegram}, {\tt Bitcoin Alpha}, {\tt Bitcoin-OTC}, and synthetic DBSM graphs. The latter are generated with $n= 2500$, $C = 5$, $\alpha_{ii} = 0.1$, $\beta_{ij} = 0.2$, with an increasing inter-cluster density $ \alpha_{ij} \in \{0.05, 0.08, 0.1\}$.\footnote{As spectral methods, except for \rete{}, cannot handle graphs with negative weights, to be able to compare our proposal to them in a setting in which the latter can be applied, in these experiments we pre-process {\tt Bitcoin-OTC} and {\tt Bitcoin Alpha} by removing any edge with a negative weight---in the tables, these datasets are denoted by a `*'.}
Following~\citet{zhang2021magnet}, in each task we reserve 15\% of the edges for testing, 5\% for validation, and use the remaining ones for training. The experiments are run with $k$-cross validation with $k=10$, preserving graph connectivity.

\begin{table*}[htb]
\scriptsize
\centering
\caption{Accuracy (\%) on datasets of the existence prediction task}
\label{tab:existence}
\begin{tabular}{cllllll}
\hline
          & \multicolumn{6}{c}{Existence prediction}                                                                                                                                                \\ \cline{2-7} 
          & Telegram                     & Bitcoin Alpha*               & Bitcoin-OTC*                 & $\alpha_{ij}=0.05$           & $\alpha_{ij}=0.08$           & $\alpha_{ij}=0.1$            \\ \hline
ChebNet   & 75.30$\pm$1.54               & 81.93$\pm$0.64               & 82.07$\pm$0.38               & 50.24$\pm$0.35               & 50.21$\pm$0.33               & 50.25$\pm$0.34               \\
GCN       & 67.88$\pm$1.39               & 81.53$\pm$0.57               & 81.65$\pm$0.35               & 50.26$\pm$0.30               & 50.24$\pm$0.26               & 50.18$\pm$0.26               \\ \hline
APPNP     & 68.52$\pm$5.76               & 81.62$\pm$0.57               & 81.02$\pm$0.51               & 60.62$\pm$0.46               & 62.61$\pm$0.64               & 63.51$\pm$1.93               \\
SAGE      & 85.36$\pm$1.27               & 82.74$\pm$0.48               & 83.28$\pm$0.65               & 60.92$\pm$0.82               & 61.50$\pm$4.05               & 62.77$\pm$1.50               \\
GIN       & 72.37$\pm$3.57               & 74.64$\pm$5.43               & 77.75$\pm$1.15               & 57.52$\pm$4.47               & 55.50$\pm$5.14               & 55.25$\pm$7.14               \\
GAT       & 78.37$\pm$2.11               & 82.60$\pm$0.43               & 83.43$\pm$0.52               & 55.97$\pm$2.58               & 54.37$\pm$0.89               & 50.24$\pm$0.35               \\ \hline
DGCN      & 82.97$\pm$2.06               & 83.13$\pm$0.61               & 83.79$\pm$0.36               & 55.41$\pm$3.09               & 55.70$\pm$5.71               & 56.15$\pm$5.65               \\
DiGraph   & 82.15$\pm$1.11               & 83.24$\pm$0.38               & \textbf{84.77$\bm{\pm}$0.83} & 59.09$\pm$3.66               & 57.64$\pm$2.35               & 58.66$\pm$3.28               \\
DiGCL     & 78.80$\pm$1.50               & 80.22$\pm$0.77               & 81.99$\pm$0.62               & 60.69$\pm$0.27               & 60.63$\pm$0.18               & 60.49$\pm$0.15               \\
MagNet    & \textbf{86.32$\bm{\pm}$1.06} & \ul{83.26$\pm$0.50}          & 84.14$\pm$0.44               & \ul{61.27$\pm$0.19}          & \ul{63.81$\pm$0.20}          & \ul{64.93$\pm$0.43}          \\ \hline
\signum{} & {\ul{84.95$\pm$0.95}}        & \textbf{83.28$\bm{\pm}$0.54} & \ul {84.71$\pm$0.39}         & \textbf{62.25$\bm{\pm}$0.31} & \textbf{64.48$\bm{\pm}$0.17} & \textbf{65.49$\bm{\pm}$0.31}

\end{tabular}
\end{table*}

\begin{table*}[htb]
\scriptsize
\centering
\caption{Accuracy (\%) on datasets of the direction prediction task}
\label{tab:direction}
\begin{tabular}{cllllll}
\hline
          & \multicolumn{6}{c}{Direction prediction}                                                                                                                                             \\ \cline{2-7} 
          & Telegram                     & Bitcoin Alpha*            & Bitcoin-OTC*                 & $\alpha_{ij}=0.05$           & $\alpha_{ij}=0.08$           & $\alpha_{ij}=0.1$            \\ \hline
ChebNet   & 78.56$\pm$3.53               & 53.86$\pm$1.15            & 50.06$\pm$1.04               & 50.13$\pm$0.30               & 50.23$\pm$0.25               & 50.13$\pm$0.30               \\
GCN       & 63.86$\pm$1.40               & 55.32$\pm$1.12            & 49.63$\pm$1.82               & 50.05$\pm$0.15               & 50.24$\pm$0.29               & 50.13$\pm$0.30               \\ \hline
APPNP     & 75.70$\pm$9.08               & 57.14$\pm$1.03            & 52.61$\pm$1.63               & 66.42$\pm$1.35               & 70.25$\pm$1.46               & 71.93$\pm$0.47               \\
SAGE      & 91.15$\pm$0.77               & 55.82$\pm$1.60            & 55.29$\pm$1.23               & 66.62$\pm$1.72               & 68.84$\pm$2.38               & 69.43$\pm$6.79               \\
GIN       & 80.77$\pm$5.01               & 56.04$\pm$1.42            & 53.31$\pm$1.58               & 60.51$\pm$6.88               & 60.87$\pm$9.50               & 57.66$\pm$9.04               \\
GAT       & 84.06$\pm$11.17              & 55.20$\pm$1.06            & 53.23$\pm$0.63               & 52.71$\pm$1.53               & 57.07$\pm$1.50               & 57.43$\pm$1.07               \\ \hline
DGCN      & {89.81$\pm$1.20}             & 56.35$\pm$0.84            & 54.06$\pm$0.90               & 55.97$\pm$2.58               & 62.64$\pm$6.91               & 65.53$\pm$6.73               \\
DiGraph   & 87.46$\pm$0.84               & {\textbf{58.62$\pm$1.09}} & 56.37$\pm$1.29        & 65.51$\pm$1.71               & 67.09$\pm$1.65               & 67.43$\pm$2.10               \\
DiGCL     & 82.98$\pm$1.72               & 55.98$\pm$0.91            & \ul{56.42$\bm{\pm}$0.59} & 67.34$\pm$0.33               & 66.92$\pm$0.26               & 66.24$\pm$0.29               \\
MagNet    & \textbf{91.65$\pm$0.79}          & 56.84$\pm$0.74            & 55.63$\pm$0.74               & \ul{68.50$\pm$0.23}          & \ul{72.01$\pm$0.33}          & \ul {73.28$\pm$0.37}         \\ \hline
\signum{} & \ul{91.20$\bm{\pm}$0.65} & \ul{56.90$\pm$0.60}       & \textbf{57.19$\bm{\pm}$0.58} & \textbf{69.10$\bm{\pm}$0.18} & \textbf{72.74\bm{$\pm$}0.23} & \textbf{73.77\bm{$\pm$}0.18} 
\end{tabular}
\end{table*}

Tables~\ref{tab:existence} and \ref{tab:direction} report the results obtained for the existence and direction prediction tasks, respectively. The tables show that, when compared to the other 10 methods, \rete{} achieves the best performance on 9 datasets out of 12 and that it achieves either the first- or the second-best performance on 12 datasets out of 12.
\rete{} is also consistently better than the state of the art on the synthetic datasets. This is likely due to the positive homogeneity property (Theorem~\ref{prop:homonegeneity}) as the synthetic datasets have a significantly wider range of weights ($[2, 1000]$) than the real-world ones (in \texttt{Telegram}, for instance, the mean and median weight is 2 and 20.7 and only 96.4\% of the weights are smaller than 100).

\subsection{Node classification}

The node classification task consists in predicting the class label to which each node belongs. 

Also for this task, we only consider graphs with nonnegative edge weights, and thus compare \rete{} not just to spatial GCNs, but also to the state-of-the-art spectral-based ones. Similarly to the previous two tasks, we will show that also for this task the wider applicability of \rete{} does not hinder its performances.

We consider the {\tt Telegram} dataset as well as the three synthetic datasets. {\tt Bitcoin-OTC} and {\tt Bitcoin Alpha} dataset are not considered as they lack label information.\footnote{Differently from~\cite{zhang2021magnet}, where {\tt Telegram} is preprocessed in such a way that it is transformed into an unweighted graph, in our experiments we retain its original weights.}
We rely on the standard 60\%/20\%/20\% split for training/validation/testing across all datasets. The experiments are run with $k$-cross validation, with $k=10$. 

\begin{table}[htb]
\scriptsize
\centering
\caption{Testing accuracy (\%) of node classification.}
\label{tab:node}
\begin{tabular}{cllll}
\hline
          & \multicolumn{4}{c}{Node classification}                                                                                   \\ \hline
          & Telegram                     & $\alpha_{ij}$ = 0.05         & $\alpha_{ij}$ = 0.08         & $\alpha_{ij}$ = 0.1          \\ \hline
ChebNet   & 61.73$\pm$4.25               & 20.06$\pm$0.18               & 20.50$\pm$0.77               & 19.98$\pm$0.06               \\
GCN       & 60.77$\pm$3.67               & 20.06$\pm$0.18               & 20.02$\pm$0.06               & 20.01$\pm$0.01               \\ \hline
APPNP     & 55.19$\pm$6.26               & 33.46$\pm$7.43               & 34.72$\pm$14.98              & 36.16$\pm$14.92              \\
SAGE      & 65.38$\pm$5.15               & 67.64$\pm$9.81               & 68.28$\pm$10.92              & 82.96$\pm$10.98              \\
GIN       & \ul{72.69$\bm{\pm}$4.62} & 28.46$\pm$8.01               & 20.12$\pm$0.20               & 20.98$\pm$8.28               \\
GAT       & 72.31$\pm$3.01          & 22.34$\pm$3.13               & 21.90$\pm$2.89               & 21.58$\pm$1.80               \\
SSSNET    & 24.04$\pm$9.29                  & \textbf{91.04$\bm{\pm}$3.60} & \ul{94.94$\pm$1.01}          & \ul{96.77$\pm$0.80}          \\ \hline
DGCN      & 71.15$\pm$6.32               & 30.02$\pm$6.57               & 30.22$\pm$11.94              & 28.40$\pm$8.62               \\
DiGraph   & 71.16$\pm$5.57               & 53.84$\pm$14.28              & 38.50$\pm$12.20              & 34.78$\pm$9.94               \\
DiGCL     & 64.62$\pm$4.50               & 19.51$\pm$1.21               & 20.24$\pm$0.84               & 19.98$\pm$0.45               \\
MagNet    & 55.96$\pm$3.59               & 78.64$\pm$1.29               & 87.52$\pm$1.30               & 91.58$\pm$1.04               \\ \hline
\signum{} & \textbf{74.23$\pm$5.24}          & \ul{87.44$\pm$0.99}          & \textbf{96.14$\bm{\pm}$0.64} & \textbf{98.60$\bm{\pm}$0.31}

\\
\end{tabular}
\end{table}

The results are reported in Table~\ref{tab:node}. \signum{} achieves notable performance on all four datasets, especially on the synthetic ones as the graph density increases, where being able to rely on both edge direction and weight information seems to be paramount for a correct node labeling. This is confirmed by the extremely poor performance of ChebNet and GCN, which ignore the edge direction.
When comparing \rete{} to MagNet, \rete{} achieves a consistently better performance of about 10\% on average. This can be ascribed to the positive homogeneity property of \rete{}, which circumvents the sign-pattern inconsistency MagNet suffers from (we recall that all these graphs have weights larger than 1), which is likely to reduce its ability to adequately propagate the information between nodes.
We note that, while SSSNET outperforms \rete{} once by about 4\%, \rete{} outperforms SSSNET by about 50\% on {\tt Telegram}. SSSNET's poor performance on this dataset is likely due to the lack of seed nodes that SSSNET needs to identify and target node classes (which are present in the DSBM graphs).


\section{Conclusions}\label{sec:conclusion}

We have extended the applicability of spectral GCNs to (directed) graphs with edges of weight unrestricted in sign by introducing the \textit{\signumlaplacian{}} matrix. Thanks to its properties, which we rigorously derived, we embedded the matrix into a generalized convolution operator which is the cornerstone of our proposed spectral GCN: \rete{}, which is first spectral GCN capable of handling (directed) graphs with weights not restricted in sign nor magnitude.
Compared with similar approaches presented in the literature, \rete{} also does not suffer from any sign-pattern inconsistencies, making it capable to handle graphs with arbitrarily large weights (also in a completely parameter-free way and without preprocessing).
Thanks to extensive numerical experiments, we have shown that, on graphs with negative weights where no other spectral GCN can be applied, \rete{}'s performance is either better or in line with more complex architectures, and that, on graphs with nonnegative weights where state-of-the-art spectral GCNs can be employed, \rete{}'s performance is consistently either the best or the second-best across all tasks.
Future works include extending the \textit{\signumlaplacian{}} to multi graphs and non-digon-free graphs without information loss (an issue shared by every Hermitian Laplacian matrix), and to temporal graphs, as well as experimenting with architectures featuring three or more convolutional layers.


\clearpage

\section*{Acknowledgements}
This work has been partially supported by
the project “ULTRA OPTYMAL - Urban Logistics and susTainable tRAnsportation: OPtimization under uncertainTY and MAchine Learning" funded by the MIUR Progetti di Ricerca di Rilevante Interesse Nazionale (PRIN) Bando 2020 - grant 20207C8T9M,
the European Union's Horizon Europe research and innovation programme under grant agreement No 101070284 - enRichMyData,
and the Alan Turing Institute under the EPSRC grants EP/W001381/1 and
EP/N510129/1.

\bibliographystyle{plainnat}
\bibliography{BIB}

\clearpage
\appendix
\allowdisplaybreaks


\section{Properties of the \signumlaplacian{}}

This section presents the proofs of theorems. 
\setcounter{theorem}{0}

\begin{theorem}
If $A \in \{0,1\}^{n \times n}$ and $q=0.25$, we have $L^{\sigma} = L^{(q)}$.
\end{theorem}
\begin{proof}
As $A \in \{0,1\}^{n \times n}$, $D$ and $\bar D$ coincide.
Let us consider
$H^{\sigma}$ and $H^{(q)}$.
For each $i,j \in V$, if $A_{ij} = 1$, we have $H^{\sigma}_{ij} = - H^{\sigma}_{ji} = 0 + \ii \frac{1}{2} = H^{(0.25)}_{ij} = - H^{(0.25)}_{ji}$; if $A_{ij} = 0$, we have $H^{\sigma}_{ij} = H^{\sigma}_{ji} = 1 + \ii 0 = H^{(0.25)}_{ij} = H^{(0.25)}_{ji}$.
Thus,
we have $L^{\sigma} = L^{(0.25)}$ and the claim follows.
\end{proof}


\begin{theorem}
$L^{\sigma}$ and $L^{\sigma}_{norm}$ are positive semidefinite.
\end{theorem}
\begin{proof}
Following the definition of the \textit{\signumlaplacian{}}, we have $\Re(L^{\sigma}) = \bar D_s - A_s \odot (ee^\top - \sgn(|A - A^\top|))$ and $\Im(L^{\sigma}) = - A_s \odot \sgn(|A| - |A^\top|)$. As $\Re(L^{\sigma})$ is symmetric and $\Im(L^{\sigma})$ is skew symmetric by construction, $L^{\sigma}$ is Hermitian.
Since $L^{\sigma}$ is Hermitian, $x^*\Im(L^{\sigma})x=0$ holds for all $x \in \CC^{n}$.
As, by construction, $\bar D_s = \Diag(|A_s| \, e)$ and $A_s$ is symmetric, the following holds for all $x \in \CC^{n}$: 

\small
$2x^*\Re(L^{\sigma})x$ 
\begin{align*}
    & = 2 \sum_{i,j=1}^n (\bar D_s)_{ij} x_i x_j^* - 2 \sum_{i,j=1}^n (A_s)_{ij}x_i x_j^*  \left(1 - \sgn(|A_{ij} - A_{ji}|)\right)\\
    & = 2 \sum_{i=1}^n (\bar D_s)_{ii} x_i x_i^* - 2 \sum_{i,j=1}^n (A_s)_{ij}x_i x_j^* \left(1 - \sgn(|A_{ij} - A_{ji}|)\right) \\
    & = 2 \sum_{i,j=1}^n |(A_s)_{ij}| |x_i|^2 - 2 \sum_{i,j=1}^n (A_s)_{ij}x_i x_j^* \left(1 - \sgn(|A_{ij} - A_{ji}|)\right)\\
    & = \sum_{i,j=1}^n |(A_s)_{ij}| |x_i|^2 + \sum_{i,j=1}^n |(A_s)_{ji}| |x_j|^2 \\ 
    & - 2 \sum_{i,j=1}^n (A_s)_{ij}x_i x_j^* \left(1 - \sgn(|A_{ij} - A_{ji}|)\right)\\
    & = \sum_{i,j=1}^n |(A_s)_{ij}| |x_i|^2 + \sum_{i,j=1}^n |(A_s)_{ij}| |x_j|^2\\ & - 2 \sum_{i,j=1}^n (A_s)_{ij}x_i x_j^* \left(1 - \sgn(|A_{ij} - A_{ji}|)\right)\\
    & = \sum_{i,j=1}^n |(A_s)_{ij}| |x_i|^2 + \sum_{i,j=1}^n |(A_s)_{ij}| |x_j|^2 \\ & - 2 \sum_{i,j=1}^n |(A_s)_{ij}| \sgn((A_s)_{ij}) x_i x_j^* \left(1 - \sgn(|A_{ij} - A_{ji}|)\right)\\
    & = \sum_{i,j=1}^n |(A_s)_{ij}| \Big(|x_i|^2 + |x_j|^2 - 2 \sgn(A_{ij}) x_i x_j^* \left(1 - \sgn(|A_{ij} - A_{ji}|)\right)\Big) \;\\
    & \geq \sum_{i,j=1}^n |(A_s)_{ij}| \Big(|x_i| - \sgn(A_{ij})|x_j| \Big)^2 \\
    & \geq 0.
\end{align*}
\normalsize
Thus, $L^{\sigma}$ is positive semidefinite.
Let us now consider the {\em normalized \laplaciano{}}, which, according to Eq.~\eqref{eq:norma}, is defined as $L^{\sigma}_\text{norm} = \rsqrt{\bar D_s} L^{\sigma} \rsqrt{\bar D_s}$. We need to show that $x^*L^{\sigma}_{\text{norm}}x \geq 0$ for all $x \in \CC^n$. Letting $y =\rsqrt{\bar D_s}x$, we have $x^*L^{\sigma}_{\text{norm}}x = x^*\rsqrt{\bar D_s} L^{\sigma} \rsqrt{\bar D_s}x = y^*L^{\sigma}y$, which is nonnegative as proven before.
\end{proof}


\begin{theorem} 
$\lambda_\text{max}(L_{\text{norm}}^{\sigma}) \leq 2$.
\end{theorem}
\begin{proof}

Let $B := \bar D_s + H^{\sigma}$. Let us show that $B$ is positive semidefinite. As $B$ is Hermitian by construction, we have $x^*\Im(L^{\sigma})x=0$.  Next, we show that $2x^*\Re(B)x \geq 0$.


\scriptsize
$2x^*\Re(B)x$ 
\begin{align*}
    & = 2 \sum_{i,j=1}^n (\bar D_s)_{ij} x_i x_j^* + 2 \sum_{i,j=1}^n (A_s)_{ij}x_i x_j^*  \left(1 - \sgn(|A_{ij} - A_{ji}|)\right)\\
    & = 2 \sum_{i=1}^n (\bar D_s)_{ii} x_i x_i^* + 2 \sum_{i,j=1}^n (A_s)_{ij}x_i x_j^* \left(1 - \sgn(|A_{ij} - A_{ji}|)\right) \\
    & = 2 \sum_{i,j=1}^n |(A_s)_{ij}| |x_i|^2 + 2 \sum_{i,j=1}^n (A_s)_{ij}x_i x_j^* \left(1 - \sgn(|A_{ij} - A_{ji}|)\right)\\
    & = \sum_{i,j=1}^n |(A_s)_{ij}| |x_i|^2 + \sum_{i,j=1}^n |(A_s)_{ji}| |x_j|^2\\
    & + 2 \sum_{i,j=1}^n (A_s)_{ij}x_i x_j^* \left(1 - \sgn(|A_{ij} - A_{ji}|)\right)\\
    & = \sum_{i,j=1}^n |(A_s)_{ij}| |x_i|^2 + \sum_{i,j=1}^n |(A_s)_{ij}| |x_j|^2\\ 
    & + 2 \sum_{i,j=1}^n (A_s)_{ij}x_i x_j^* \left(1 - \sgn(|A_{ij} - A_{ji}|)\right)\\
    & = \sum_{i,j=1}^n |(A_s)_{ij}| |x_i|^2 + \sum_{i,j=1}^n |(A_s)_{ij}| |x_j|^2 \\
    & + 2 \sum_{i,j=1}^n |(A_s)_{ij}| \sgn((A_s)_{ij}) x_i x_j^* \left(1 - \sgn(|A_{ij} - A_{ji}|)\right)\\
    & = \sum_{i,j=1}^n |(A_s)_{ij}| \Big(|x_i|^2 + |x_j|^2 + 2 \sgn(A_{ij}) x_i x_j^* \left(1 - \sgn(|A_{ij} - A_{ji}|)\right)\Big) \;\\
    & \geq \sum_{i,j=1}^n |(A_s)_{ij}| \Big(|x_i|^2 + |x_j|^2 \Big) \\
    & \geq 0.
\end{align*}
\normalsize

Thus, the normalized version of $B$ satisfies

\[
x^*B_{\text{norm}}x = x^*\rsqrt{\bar D_s} B \rsqrt{\bar D_s}x = y^*By\geq  0.
\]

We have proved that $x^*B_{\text{norm}}x$ is positive semidefinite. Hence, the following holds:
\begin{align*}
    & x^*B_{\text{norm}}x \geq 0 \\
    & x^*\left( I + \rsqrt{D}H^{\sigma}\rsqrt{D} \right)x \geq 0 \\
    & -x^* \rsqrt{D} H^{\sigma} \rsqrt{D}x \leq x^*x \\
    & x^*Ix - x^* \rsqrt{D}H^{\sigma} \rsqrt{D}x \leq 2x^*x \\
    & \frac{x^* L^{\sigma}_{\text{norm}}x}{x^*x} \leq 2.
\end{align*}
Due to the Courant-Fischer theorem applied to $L^{\sigma}_{\text{norm}}$, we have:
$$
 \lambda_{\text{max}}= \max_{x \neq 0} \frac{x^*L^{\sigma}_{\text{norm}}x}{x^*x}.
$$
Thus, $\lambda_{\text{max}} \leq 2$ holds.
\end{proof}


\begin{theorem}
Given a constant $\alpha \in \RR^+$, $L^{\sigma}$ satisfies the following positive homogeneity property:
$$
L^{\sigma}(\alpha A) = \alpha L^{\sigma}(A),
$$
where $L^{\sigma}(\alpha A)$ and $L^{\sigma}(A)$ are the \laplaciano{} matrices of a directed graph with, respectively, adjacency matrix $\alpha A \in \RR^{n \times n}$ and $A \in \RR^{n \times n}$.
\end{theorem}
\begin{proof}
Let $H^{\sigma}(X)$ and $\bar D(X)$ be the $H^\sigma$ and $\bar D$ matrices of a directed graph with adjacency matrix $X \in \RR^{n \times n}$.
We have:

\small
$H^\sigma(A \odot B)$
\begin{align*}
& = \left(\frac{A \odot B + A^\top \odot B^\top}{2} + \ii \frac{A \odot B + A^\top \odot B^\top}{2}\right) \odot \\
& \left((ee^\top - \sgn(|A - A^\top|)) + \ii \, \sgn(|A| - |A^\top|)\, \right) =\\
& = \left(\alpha \frac{(A + A^\top)}{2} + \ii \frac{\alpha(A + A^\top)}{2}\right) \odot \\ & \left((ee^\top - \sgn(|A - A^\top|)) + \ii \, \sgn(|A| - |A^\top|)\right) =\\
& = \left(\frac{A + A^\top}{2} + \ii \frac{A + A^\top}{2}\right) \odot \\
& \left((ee^\top - \sgn(|A - A^\top|)) + \ii \, \sgn(|A| - |A^\top|)\right) \odot B = \\
& = H^{\sigma}(A) \odot B.
\end{align*}
\normalsize

$\bar D(A \odot B) = \Diag(|A \odot B| e) = \Diag(|A| e) = \bar D(A) \odot B$ (the latter by construction of $B$). The claim follows.
\end{proof}

\begin{theorem}
Consider a weighted directed graph $G=(V,E)$ without pairs of antiparallel edges (digons). Given a directed edge $(i,j) \in E$ of weight $w_{ij}$, let $G'=(V,E)$ be a copy of $G$ obtained by reversing the direction of $(i,j)$ into $(j,i)$ and flipping the sign of its weight by letting $w_{ji} = -w_{ij}$. Let $L^{\sigma}(G)$ and $L^{\sigma}(G')$ be the $L^{\sigma}$ matrix defined on $G$ and $G'$, respectively. $L^{\sigma}(G) = L^{\sigma}(G')$ holds.
\end{theorem}
\begin{proof}
Let $A_G$ and $A_{G'}$ be the adjacency matrices of $G$ and $G'$.
Let $H^\sigma(X)$ and $A_s(X)$ be the $H^\sigma$ and $A_s$ matrices defined for a graph with adjacency matrix $X$.
$\Re(H^\sigma(G)) = \Re(H^\sigma(G'))$ holds since both $A_G - A_G^\top$ and $A_{G'} - A_{G'}^\top$ are nonzero in positions $i,j$ and $j,i$ and, thus, $(ee^\top - \sgn(|A_G - A_G^\top|))$ and $(ee^\top - |\sgn(A_{G'} - A_{G'}^\top)|)$ are both equal to 0 in these positions.
%
%
To see that $\Im(H^\sigma(G)) =\Im(H^\sigma(G'))$, we observe that $A_s(A_G)_{ij} = -A_s(A_{G'})_{ij}$ and $A_s(A_G)_{ji} = -A_s(A_{G'})_{ji}$, but also that $\sgn(|A_G| - |A_G^\top|)_{ij} = -\sgn(|A_{G'}| - |A_{G'}^\top|)_{ij}$ and that $\sgn(|A_G| - |A_G^\top|)_{ji} = -\sgn(|A_{G'}| - |A_{G'}^\top|)_{ji}$.
Thus, the two differences in sign cancel out and the claim follows.
\end{proof}

\paragraph{Further observation} The results presented in this paper still hold if the imaginary part of $H^{\sigma}$ is multiplied by any nonnegative real constant $\epsilon > 0$. If $A \in \{0,1\}^{n \times n}$, by choosing $\epsilon = \sqrt{3}$, $L^{\sigma}$ coincides with the Hermitian matrix ``of the second kind'' proposed in~\cite{MOHAR2020343} in the context of algebraic graph theory.

\section{Sign-pattern inconsistency of $L^{(q)}$}\label{sub:sing-flipping}

We highlighted that the \textit{Magnetic Laplacian}, $L^{(q)}$, exhibits a crucial sign-pattern inconsistency. 
Indeed, while, for unweighted graphs, $L^{(q)}$ encodes the directional information of the edges in the sign of the imaginary part of $H^{(q)}$, this is not necessarily the case for weighted graphs as the sign pattern of $H^{(q)}$ can change drastically by just scaling the weights of the graph by a positive constant.



To better illustrate this, we introduce the following example.
Consider a directed graph $G=(V,E)$ with $V = \{1,2\}$ and $E = \{(1,2)\}$. Let us assume that the weight of the $(1,2)$ edge can take one of the following four values: 0.8, 2, 5, 36 and let $q = 0.25$. 
Although the direction of the edge $(1,2)$ does not change, based on the magnitude of the weight, we observe four different scenarios. 

\begin{enumerate}

\item $A = \begin{bmatrix} 
	0 & 0.8 \\
	0 & 0 \\
	\end{bmatrix}$,
	$A_s = 	\begin{bmatrix} 
	0 & 0.4 \\
	0.4 & 0 \\
	\end{bmatrix}$, and
    $H^{(0.25)} = \begin{bmatrix} 
	0 & 0.4 \\
	0.4 & 0 \\
	\end{bmatrix} \odot \begin{bmatrix} 
	1 & 0.31 \\
	0.31 & 1 \\
	\end{bmatrix} + \ii\begin{bmatrix} 
	0 & 0.40 \\
	0.40 & 0 \\
	\end{bmatrix} \odot \begin{bmatrix} 
	0 & 0.95 \\
	-0.95 & 0 \\
      \end{bmatrix}$.
      We have $\sgn(\Im(H^{(0.25)})_{12}) = -\sgn(\Im(H^{(0.25)})_{21})  = \sgn(A_{12})$ and, thus, the sign of the imaginary part of $H^{(0.25)}$ encodes the direction of the edge, while $\Re(H^{(0.25)})_{12} = \Re(H^{(0.25)})_{21} \neq 0$.

\item $A = \begin{bmatrix} 
	0 & 2 \\
	0 & 0 \\
	\end{bmatrix}$,
	$A_s = 	\begin{bmatrix} 
	0 & 1 \\
	1 & 0 \\
	\end{bmatrix}$, and
    $H^{(0.25)} = \begin{bmatrix} 
	0 & 1 \\
	1 & 0 \\
	\end{bmatrix} \odot \begin{bmatrix} 
	1 & -1 \\
	-1 & 1 \\
	\end{bmatrix} + \ii\begin{bmatrix} 
	0 & 1 \\
	1 & 0 \\
	\end{bmatrix} \odot \begin{bmatrix} 
	0 & 0 \\
	0 & 0 \\
      \end{bmatrix}$.
      We have $\Im(H^{(0.25)})_{12} = \Im(H^{(0.25)})_{21} = 0$ and, thus, the sign of the imaginary part of $H^{(0.25)}$ does not encode at all the direction of the edge. Furthermore, we note that $\sgn(\Re(H^{(0.25)})_{12}) = \sgn(\Re(H^{(0.25)})_{21}) \neq \sgn(A_{12})$. Consequently, the matrix $H^{(0.25)}$ represents the graph as an undirected graph with a negative weight.

\item $A = \begin{bmatrix} 
	0 & 5 \\
	0 & 0 \\
	\end{bmatrix}$,
	$A_s = 	\begin{bmatrix} 
	0 & 2.5 \\
	2.5 & 0 \\
	\end{bmatrix}$, and
	$  H^{(0.25)}= \begin{bmatrix} 
	1 & 2.5 \\
	2.5 & 1 \\
	\end{bmatrix} \odot \begin{bmatrix} 
	1 & 0 \\
	0 & 1 \\
	\end{bmatrix} + \ii\begin{bmatrix} 
	0 & 2.5 \\
	2.5 & 0 \\
	\end{bmatrix} \odot \begin{bmatrix} 
	0 & 1 \\
	-1 & 0 \\
      \end{bmatrix}$.
     We have $\sgn(\Im(H^{(0.25)})_{12}) = - \sgn(\Im(H^{(0.25)})_{21}) = \sgn(A_{12})$; thus, the sign of the imaginary part of $H^{(0.25)}$ encodes the direction of the edge $(1,2)$ consistently with $A$, while $\Re(H^{(0.25)})_{12} = \Re(H^{(0.25)})_{21} = 0$;

      


\item $A = \begin{bmatrix} 
	0 & 36 \\
	0 & 0 \\
	\end{bmatrix}$,
	$A_s = 	\begin{bmatrix} 
	0 & 18 \\
	18 & 0 \\
	\end{bmatrix}$, and
    $H^{(0.25)} = \begin{bmatrix} 
	0 & 18 \\
	18 & 0 \\
	\end{bmatrix} \odot \begin{bmatrix} 
	1 & 1 \\
	1 & 1 \\
	\end{bmatrix} + \ii\begin{bmatrix} 
	0 & 18 \\
	18 & 0 \\
	\end{bmatrix} \odot \begin{bmatrix} 
	0 & 0 \\
	0 & 0 \\
      \end{bmatrix}$.
      We have $\Im(H^{(0.25)})_{12} = \Im(H^{(0.25)})_{21} = 0$ and, thus, the sign of the imaginary part of $H^{(0.25)}$ does not encode the direction of the edge, while $\sgn(\Re(H^{(0.25)})_{12}) = \sgn(\Re(H^{(0.25)})_{21}) = \sgn(A_{12})$.
      Consequently, the matrix $H^{(0.25)}$ represents the graph as an undirected graph with a positive weight.

      
\end{enumerate}


\section{Flow-based pre-processing}\label{sec:net-flow}


If applied an edge at a time, Theorem~\ref{prop:sign-direction-invariant} can be used to transform a given directed graph with digons into a multigraph. For applications where the graph entails a flow-like relationship, it is then natural to aggregate every pair of parallel edges thus obtained into a single edge by summing their weights, thereby obtaining a (simple) weighted graph. In more details, consider two antiparallel edges $(i, j)$ and $(j, i)$ with different weights $\left(w_{ij} \neq w_{ji} \right)$. 
By applying Theorem~\ref{prop:sign-direction-invariant} to the $(i,j)$ arc, we reverse its direction into $(j, i)$ and flip the sign of its weight, thus obtaining the edge $(j,i)$ of weight $w_{ji} := -w_{ij}$.
As the graph already contains an $(j,i)$ arc, the graph is turned into a multigraph.
If the graph models a flow-like relationship, it is reasonable to collapse such a pair of parallel edges into a single edge of weight equal to $w_{ji} := -w_{ij} + w_{ji}$.
We carry out this operation as a pre-processing activity for each task except for the link sign prediction task, whose datasets do not entail flow-like information.


%

In the following, we report a quantitative example to show the positive impact of this technique. 
In more detail, we consider two scenarios:
\begin{enumerate}
    \item The flow-based pre-processing is not applied to the graph. As a consequence, some information related to the topology of the graph is lost.
    \item The flow-based pre-processing is applied to the graph. No information is lost.
\end{enumerate}
Consider a graph with a pair of antiparallel edges represented by the adjacency matrix $ A =
	\begin{bmatrix} 
	0 & 1 \\
	-1 & 0 \\
	\end{bmatrix}$.

\begin{enumerate}
\item If we do not apply the flow-based pre-processing, we have $A_s = 	
\begin{bmatrix} 
	0 & 0 \\
	0 & 0 \\
	\end{bmatrix}$. 
Thus, $L^{\sigma}$ (but also $L^{(q)}$) fails to represent the graph as the entire pair of antiparallel edges is lost.
\item If we apply the flow-based pre-processing to the graph (not applicable for $L^{(q)}$), we obtain the following new adjacency matrix: $A_\text{new} =
	\begin{bmatrix} 
	0 & 2 \\
	0 & 0 \\
	\end{bmatrix}
$; thus, we have $A_{s_{\text{new}}} =
	\begin{bmatrix} 
	0 & 1 \\
	1 & 0 \\
	\end{bmatrix}$. Thanks to this, $L^\sigma$ consistently represent a graph with one edge, the direction of which is encoded in the imaginary part of $L^{\sigma}$.
	
\end{enumerate}

\section{Details on Datasets}\label{sub:dataset}

All the datasets we considered can be obtained from our code.

\paragraph{Real-world datasets.} We test \rete{} on six real-world datasets: {\tt Bitcoin-OTC} and {\tt Bitcoin Alpha}~\citep{7837846}; {\tt Slashdot} and {\tt Epinions}~\citep{leskovec2010signed}; {\tt WikiRfa}~\citep{west2014exploiting}; and {\tt Telegram}~\citep{bovet2020activity}. 
The first two datasets, {\tt Bitcoin-OTC} and {\tt Bitcoin Alpha}, come from exchange operations: Bitcoin-OTC and Bitcoin Alpha. Both of these exchanges allow users to rate the others on a scale of $-10$ to $+10$ (excluding $0$). According to the OTC guidelines, scammers should be given a score of $-10$, while at the other end of the spectrum, $+10$ means full trust. Other evaluation values have intermediate meanings. Therefore, these two exchanges explicitly lead to a graph with weights unrestricted in sign. 
The other two datasets are {\tt Slashdot} and {\tt Epinions}. The first comes from a tech news website with a community of users. The website introduced Slashdot Zoo features that allow users to tag each other as friend or foe. The dataset represents a signed social network with friend $(+1)$ and enemy $(-1)$ labels. {\tt Epinions} is an online who-trust-who social network of a consumer review site (\textit{Epinions.com}). Site members can indicate their trust or distrust of other people's reviews. The network reflects people's views on others.
{\tt WikiRfa} is a collection of votes given by Wikipedia members collected from 2003 to 2013. 
Indeed, any Wikipedia member can vote for support, neutrality, or opposition to a Wikipedia editor's nomination for administrator. 
This leads to a directed, 
multigraph (unrestricted in sign) in which nodes represent Wikipedia members and edges represent votes, which is then transformed into a simple graph by condensing any parallel edges into a single edge of weight equal to the sum of the weights of the original edges. The graph features a higher number of nodes and edges than the one proposed in~\cite{huang2021sdgnn}. 
In these five datasets, the classes of positive and negative edges are imbalanced (see Table \ref{tab:stats}). 
The last dataset is {\tt Telegram}, an influence network that analyses the interactions and influences between distinct groups and actors who associate and propagate political ideologies. This is a pairwise-influence network between 245 Telegram channels with 8912 links. The labels are generated following the method discussed in~\cite{bovet2020activity}, with a total of four classes.


\begin{table*}[htb]
\centering
\caption{Statistics of the six datasets}
\label{tab:stats}
\begin{tabular}{lrrrrccc}
\hline
Data set            & $n$     & $|\varepsilon^+|$ & $|\varepsilon^-|$ & \% pos & Directed   & Weighted   & Density \\ \hline
{\tt Telegram}      & 245     & 8,912             & 0                 & 100.00 & \checkmark & \checkmark & 14.91\% \\
{\tt Bitcoin-Alpha} & 3,783   & 22,650            & 1,536             & 93.65  & \checkmark & \checkmark & 0.17\%  \\
{\tt Bitcoin-OTC}   & 5,881   & 32,029            & 3,563             & 89.99  & \checkmark & \checkmark & 0.10\%  \\
{\tt WikiRfA}       & 11,381  & 138,143           & 39,038            & 77.97  & \checkmark & \checkmark & 0.14\%  \\
{\tt Slashdot}      & 82,140  & 425,072           & 124,130           & 77.70  & \checkmark & \xmark     & 0.01\%  \\
{\tt Epinion}       & 131,828 & 717,667           & 123,705           & 85.30  & \checkmark & \xmark     & 0.01\%  \\ \hline
\end{tabular}
\end{table*}


\paragraph{Synthetic dataset.} The synthetic set of graphs are generated via a direct stochastic block model (DSBM) with (unlike in~\cite{zhang2021magnet}) edge weights in the range $\NN \cap [2, 1000]$. 
In detail, in DSBM we define a number of nodes $n$ and a number of clusters $C$ which partition the vertices into communities of equal size. We define a collection of probabilities $\{\alpha_{ij}\}_{1 \leq i,j \leq C}$, where $0 \leq \alpha_{ij} \leq 1$ with $\alpha_{ij} = \alpha_{ji}$, to define the probability that an undirected edge be generated between a node $u$ and a node $v$ that belong to two different clusters, i.e., $u \in C_i$ and $v \in C_j$, and $\alpha_{ii}$ is the probability that an undirected edge is generated between two nodes in the same cluster. As the generated graph is undirected, 
we follow~\cite{zhang2021magnet} and introduce a rule to transform the graph from undirected to directed: we define a collection of probabilities $\{\beta_{ij}\}_{1 \leq i,j \leq C}$, where $0 \leq \beta_{ij} \leq 1$ such that $\beta_{i,j} + \beta_{j,i} = 1$. Each edge $\{u, v\}$ is assigned a direction using the rule that the edge points from $u$ to $v$ with probability $\beta_{ij}$ if $u \in C_i$ and $v \in C_j$, and points from $v$ to $u$ with probability $\beta_{ji}$.
For the characteristics of the loss function present in the SSSNET model, we set 10\% of the nodes per class of the graph as seed nodes.

\section{Experiment Details}



\paragraph{Hardware.} The experiments were conducted on 2 different computers: one with 1 NVIDIA Tesla T4 GPU, 380 GB RAM, and Intel(R) Xeon(R) Gold 6238R CPU @ 2.20GHz CPU, and the other with 1 NVIDIA TITAN Xp GPU, 80 GB RAM, and Intel(R) Xeon(R) Silver 4114 CPU @ 2.20GHz CPU.

\paragraph{Model Settings.} We train all the models considered in this paper with a maximum of 3000 epochs and early stop if the validation error does not decrease after 500 epochs for both node classification and link prediction tasks. 
As in \cite{zhang2021magnet}, one dropout layer with a probability of $0.5$ is created before the last layer. We set the parameter $K = 1$  for ChebNet, MagNet, and \rete{}. 
A hyperparameter optimization procedure is adopted to identify the best set of parameters for each model. 
We tune the number of filters in $\{16, 32, 64\}$ for the graph convolutional layers for all models except for DGCN. We set for both node classification and link prediction a learning rate of $10^{-3}$. For link sign prediction task, the learning rate is set in $\{10^{-2}, 5 \cdot 10^{-3}, 10^{-3}\}$. We employ Adam as the optimization algorithm, and set weight decays (regularization hyperparameter) to $5 \cdot 10^{-4}$ to prevent overfitting. 

Some further details are reported in the following:
\begin{itemize}
    \item The coefficient $q$ for MagNet is chosen in $\{0.01, 0.05, 0.1, 0.15, 0.2, 0.25\}$.
    \item The coefficient $\alpha$ for PageRank-based models (APPNP and DiGraph) is chosen in $\{0.05, 0.1, 0.15, 0.2\}$.
    \item For APPNP, we set $K = 10$ for node classification (parameter suggested in \cite{klicpera2018predict}), and select $K$ in $\{1, 5, 10\}$ for link prediction.
    \item For GAT, we adopt a number of heads in $\{2, 4, 8\}$.
    \item DGCN is somewhat different from the other networks because it requires generating three matrices of order proximity, i.e., first-order proximity, second-order \textit{in-degree proximity} and second-order \textit{out-degree proximity}. For this network, the number of filters for each channel is searched in $\{5, 15, 30\}$ for node classification and link prediction.
    \item In GIN, the parameter $\epsilon$ is set to 0 for both tasks.
    \item In SSSNET, parameters $\gamma_s$ and $\gamma_t$ are set to 50 and 0.1 respectively.
    \item In ChebNet and GCN, the symmetrized adjacency matrix $A_s = \frac{A + A^\top}{2}$ is used.
    \item For DiGCL, we select the {\em Pacing function} in [{\em linear}, {\em exponential}, {\em logarithmic}, {\em fixed}]. We also adopt two different configuration: \textit{i)} $\tau = 0.4$, {\em drop feature rate 1} = 0.3 and {\em drop feature rate 2} = 0.4, and \textit{ii)} $\tau = 0.9$, {\em drop feature rate 1} = 0.2 and {\em drop feature rate 2} = 0.1.
\end{itemize}

\paragraph{Link prediction.} In these tasks, we define the feature matrix $X \in \RR^{n \times 2}$ in such a way that, for each node $i \in V$, $X_{i1}$ is the in-degree of node $i$ and $X_{i2}$ is the node's out-degree. This is done to allow the models to learn structural information directly from the adjacency matrix. In particular, for the sign link prediction task, we use in-degree and out-degree by computing the absolute value of their edge weights.

\paragraph{Node classification.} In this task, for the {\tt Telegram} dataset we retain the dataset's original features, whereas, for the synthetic datasets, we create them via the in-degree and out-degree vector as explained before.

\section{Complexity of SigMaNet}\label{sub:complexity}

Assuming, as done in our experiments, that \rete{} features two graph-convolutional layers with $f_1$ and $f_2$ filters, each defined as in Equation~\eqref{eq:convolution} and $c$ features per node, the complexity of \rete{} is $O(nc (n + f_1) + nf_1 (n + f_2) + m^\text{train} f_2 d)$ for link prediction (sign/direction/existence task) and $O(nc (n + f_1) + nf_1 (n + f_2) + n f_2 d))$ for node classification, where $m^\text{train}$ is the number of edges in the training set and $d$ is the number of classes. 
The detailed calculations for complexity are as follows:

\begin{enumerate}
    \item The equation $\tilde D^{-1/2} H^\sigma \tilde D^{-1/2} \in \mathbb{C}^{n \times n}$ is computed in $O(|H^\sigma|) = O(n^2)$ in pre-processing (once per graph, independently of the node features).
    \item The first convolutional layer requires, $O(n^2 c + n c f_1 + n f_1) = O(nc (n + f_1))$ due to 3 operations:
    \begin{enumerate}
        \item It multiplies $\tilde D^{-1/2} H^\sigma \tilde D^{-1/2}$ by the node-feature matrix $X \in \mathbb{C}^{n \times c}$, obtaining $P^{11} \in \mathbb{C}^{n \times c}$ in $O(n^2 c)$ (assuming matrix multiplications require cubic time);
        \item It multiplies $P^{11}$ by the weight matrix $\Theta \in \mathbb{R}^{c \times f_1}$, obtaining $P^{12} \in \mathbb{C}^{n \times f_1}$ in $O(n c f_1)$;
        \item It applies the activation function $\phi$ to $P^{12}$ in $O(n f_1)$, resulting in $P^{13} \in \mathbb{C}^{n \times f_1}$.
    \end{enumerate}
    \item The second convolutional layer carries out similar operations with $c \rightarrow f_1$ and $f_1 \rightarrow f_2$, building $P^{23} \in \mathbb{C}^{n \times f_2}$ in $O(n^2 f_1 + n f_1 f_2 + n f_2) = O(nf_1 (n + f_2))$.
    \item Parts I and II of the unwind layer require a diversification based on the task to be solved:
    \begin{enumerate}
        \item  In the link sign/existence/direction tasks, $O(m^\text{train} f_2 + m^\text{train} f_2 d) = O(m^\text{train} f_2 d)$ due to 2 operations:
        \begin{enumerate}
            \item Unwinding $P^{23}$ into $U^{IL} \in \mathbb{R}^{m^\text{train} \times 4 f_2}$ in (assuming random access) $O(m^\text{train} f_2)$.
            \item Multiplying (linear layer) $U^{IL}$ by $W^{IIL} \in \mathbb{R}^{4 f_2 \times d}$ to obtain $U^{IIL} \in \mathbb{R}^{m^\text{train} \times d}$ in $O(m^\text{train} f_2 d)$.
        \end{enumerate}
        \item  In the node classification task, $O(n f_2 + n f_2 n f_2) = O(n f_2 n f_2)$ due to 2 operations:
        \begin{enumerate}
            \item Unwinding $P^{23}$ into $U^{IN} \in \mathbb{R}^{n \times 2 f_2}$ in $O(n f_2)$.
            \item Applying a 1D convolution with a 0-dimensional kernel between $U^{IN}$ and $C^{IIN} \in \mathbb{R}^{2 f_2 \times d}$, calculating $U^{IIN} \in \mathbb{R}^{n \times d}$ in $O(n f_2 d)$.
        \end{enumerate}
    \end{enumerate}
    \item The Softmax activation function requires linear time w.r.t. its input size, thus not playing any role in the analysis.
\end{enumerate}

\end{document}